\documentclass[letterpaper]{article}    %
\usepackage{aaai23}                     %
\usepackage[utf8]{inputenc}
\usepackage{times}                      %
\usepackage{helvet}                     %
\usepackage{courier}                    %
\usepackage[hyphens]{url}               %
\usepackage{graphicx}                   %
\usepackage[]{biblatex}
\usepackage{caption}                    %
\usepackage[colorlinks=true,allcolors=blue,pdfborder={0 0 0}]{hyperref}

\graphicspath{ {./FIGURES/} }

\usepackage{amsfonts}
\usepackage{amsmath}
\usepackage{amsthm}
    \theoremstyle{plain}
    \newtheorem{theorem}{Theorem}
    \newtheorem{lemma}{Lemma}
    \newtheorem{definition}{Definition}
\usepackage{enumitem}

\newcommand*\widebar[1]{%
  \hbox{%
    \vbox{%
      \hrule height 0.1pt %
      \kern0.4ex%
      \hbox{%
        \kern-0.2em%
        \ensuremath{#1}%
        \kern-0.2em%
      }%
    }%
  }%
}

\renewcommand\P{\mathbf{P}}
\newcommand\E{\mathbf{E}}
\newcommand\comp[1]{\widebar{#1}}
\newcommand\CircPref{E}
\newcommand\Index[1]{D_{#1}}
\newcommand\UniquePref{I^*}
\newcommand\MaxPref{r}

\usepackage{todonotes}
\usepackage{color}

\title{Error in the Euclidean Preference Model}

\author {
    Luke Thorburn,
    Maria Polukarov,
    Carmine Ventre
}

\affiliations {
    King's College London\\
    \{luke.thorburn, maria.polukarov, carmine.ventre\}@kcl.ac.uk
}

\begin{document}

\maketitle

\begin{abstract}
    Spatial models of preference, in the form of vector embeddings, are learned by many deep learning and multiagent systems, including recommender systems. Often these models are assumed to approximate a Euclidean structure, where an individual prefers alternatives positioned closer to their ``ideal point'', as measured by the Euclidean metric. However, \textcite{bogomolnaia2007} showed that there exist ordinal preference profiles that cannot be represented with this structure if the Euclidean space has two fewer dimensions than there are individuals or alternatives. We extend this result, showing that there are situations in which almost all preference profiles cannot be represented with the Euclidean model, and derive a theoretical lower bound on the expected error when using the Euclidean model to approximate non-Euclidean preference profiles. Our results have implications for the interpretation and use of vector embeddings, because in some cases close approximation of arbitrary, true ordinal relationships can be expected only if the dimensionality of the embeddings is a substantial fraction of the number of entities represented.
\end{abstract}

\section{Introduction}
\label{sec:introduction}

Accurate modelling of preferences is critical to the safe and efficient deployment of AI in multiagent systems \cite{christian2020}. Indeed, if the preference models used by AI agents are not sufficiently expressive to be capable of representing human preferences, they will be at least partially ``mistaken'' about what humans want and, consequently, may take actions that cause harm. For example, a recommender system built on inaccurate preference models may make recommendations that bias perceptions of politics \cite{huszar2021}, contribute to low self esteem \cite{faelens2021}, or encourage unsafe medical interventions \cite{johnson2021}.

Spatial models of preference, in the form of vector embeddings, are widely used in deep learning systems. In user-facing contexts such as recommender systems, user embeddings contain information that might be described literally as preferences \cite{pan2019}. More generally, embeddings that are intended to capture degrees of similarity between elements of a set can often be viewed as a spatial model of preference. For example, each word in a language could be considered to have a ``preference'' over all other words, ``preferring'' those with similar meanings. Spaces of word embeddings \cite{almeida2019} can thus be viewed as models of preference. Spatial models of (literal) preference are also used in the fields of political science \cite{hinich2008}, social choice theory \cite{miller2015} and opinion dynamics \cite{aydogdu2017}, as well as in preference aggregation software such as \emph{Polis} \cite{small2021} and Twitter's \emph{Community Notes} feature~\cite{twitter2022}.

A canonical spatial model is the Euclidean model, where both individuals and alternatives are represented as points in Euclidean space, and each individual prefers alternatives positioned nearer to them, as measured by the standard Euclidean metric \cite{bogomolnaia2007}. A preference profile of $I$ individuals over $A$ alternatives is said to be $d$-Euclidean if it can be represented with a $d$-dimensional Euclidean model. 

The Euclidean model is often used explicitly, and even spatial preference models that are not strictly Euclidean are often assumed to have an approximately Euclidean structure. For example, embeddings in deep learning systems are usually compared using cosine similarity \cite[Ch. 6.4]{jurafsky2020}, which induces the same ordinal relationships as the Euclidean metric when applied to normalized vectors. Given the importance of accurately representing human preferences and the prevalence of Euclidean preference models, it is important to understand their limitations.

In this paper, we assume a ``ground truth'' preference structure where each individual's preference is a strict order over the available alternatives. Consider the expressiveness of the Euclidean preference model relative to this ordinal model. There are three questions one might ask, of increasing informativeness. For fixed positive integers $I$, $A$ and $d$:
    \begin{enumerate}[topsep=0pt, itemsep=0pt]
        \def\labelenumi{\arabic{enumi}.}
    
        \item Are there any preference profiles of $I$ individuals over $A$ alternatives that are not $d$-Euclidean?
    
        \item What proportion of such profiles are not $d$-Euclidean?
    
        \item How large is the expected error when approximating arbitrary preferences with a $d$-dimensional Euclidean model?

    \end{enumerate}
Question 1 was answered by \textcite{bogomolnaia2007}, who showed that when $d < \min\{I-1, A-1\}$, there exist profiles of $I$ preferences over $A$ alternatives that cannot be represented in a $d$-dimensional Euclidean model. In this work, we address questions 2 and 3.

\paragraph{Contributions.}
We extend the results of \textcite{bogomolnaia2007} by proving a series of theoretical bounds. Intuitively, these bounds indicate that when the dimensionality of the Euclidean model is small relative to the number of individuals and the number of alternatives, it is possible that almost all preference profiles cannot be represented. Further, we describe conditions under which only a minute proportion of all possible preferences can be simultaneously represented, unless the dimensionality of the Euclidean model is almost as large as the number of individuals or the number of alternatives. Our final bound is on the expected error when approximating a randomly chosen preference in a Euclidean model of given dimensionality, where we formalize error as the number of adjacent swaps required to transform the nearest representable preference into the true preference that is being approximated.

Our results have implications for the interpretation and use of vector embeddings, because there are situations in which close approximation of arbitrary, true preferences (or preference-like data) is possible only if the dimensionality of the embeddings is a substantial fraction of the number of individuals or alternatives. In these situations, our theoretical bounds can inform the choice of the dimensionality of vector embeddings by quantifying the expected error when the dimensionality is too small. 

\paragraph{Related Work.}
This paper builds on the work of \textcite{bogomolnaia2007}; we state the relevant results in Section \ref{sec:three-questions}. Important context is also provided by \textcite{peters2016} who showed that the problem of determining whether a preference profile is $d$-Euclidean is, in general, NP-hard, and that some ordinal preference profiles require exponentially many bits to be represented in the Euclidean model. These hardness results rule out some of the most obvious approaches to evaluating the expressiveness of the Euclidean model, because for a given ordinal profile it is usually not feasible to check whether it is $d$-Euclidean, or to compute its best approximation with a $d$-dimensional Euclidean model \cite{tydrichova2023}. That said, the computational task of approximating an ordinal preference profile with a $d$-dimensional Euclidean model is known as \textit{multidimensional unfolding}, and has a number of proposed algorithms \cite{bennett1960,elkind2014,luaces2015}.

A related line of work discusses the limitations of the Euclidean preference model from the perspective of measurement theory and psychometric validity. For example, \textcite{eguia2013} questions the validity of the utility functions, indifference curves and separability of preferences that are implied by the Euclidean model. \textcite{henry2013} present analysis suggesting that the Euclidean model is not consistent with real world voting data. A number of works suggest the $L^1$ metric may be more appropriate than the $L^2$ metric in spatial preference models \cite{humphreys2010,rivero2011,eguia2011,eguia2013}.

A general review of structured preference models is given by \textcite{elkind2017}.

\paragraph{Outline.}

Section \ref{sec:preliminaries} introduces our notation and definitions. Section \ref{sec:three-questions} presents answers to the three questions listed above, including our theoretical bounds. Implications are discussed in Section~\ref{sec:conclusions}.

\section{Preliminaries}
\label{sec:preliminaries}

Let $A\in\mathbb{N}$ be the number of alternatives and $I\in \mathbb{N}$ be the number of individuals. Each individual is assumed to have a \emph{preference}: a strict order (ranking without ties) over all $A$ alternatives. The preference of individual $i$ is denoted $\pi_i$ (a ranked list) or $>_i$ (the corresponding order relation), so if $a$ and $b$ are alternatives, $a >_i b$ means individual $i$ prefers $a$ to $b$. The list of all preferences represented in the population of $I$ individuals is called a \emph{profile}. For given values of $A$ and $I$, the set of all possible profiles is denoted $\mathcal{P}_{A,I}$. The number of \emph{unique} preferences in a profile is denoted $\UniquePref$.

In the $d$-dimensional Euclidean preference model, both alternatives and individuals are represented as points in $d$-dimensional Euclidean space, denoted $\mathbb{R}^d$. We refer to these points as the \emph{location} of each alternative and the \emph{ideal point} for each individual. Each individual's preference is determined by the distance between their ideal point and the location of each alternative. Nearer alternatives are preferred, as measured by the standard Euclidean metric.

    \begin{definition}[Euclidean preferences]
        A profile $\Pi\in\mathcal{P}_{A,I}$ is \emph{$d$-Euclidean} if there exist points $x^a\in\mathbb{R}^d$ for all $a\in\{1,\dots,A\}$, and $w^i\in\mathbb{R}^d$ for all $i\in\{1,\dots,I\}$ such that, for all alternatives $a,b$ and individuals $i$, $a>_ib\Leftrightarrow \|x^a-w^i\| < \|x^b-w^i\|$ (using the standard Euclidean norm).
    \end{definition}

Note that while some versions of the Euclidean preference model are used to represent utilities or \emph{cardinal} preferences, in our definition the Euclidean model only describes \emph{ordinal} preferences. We focus on strict orders for simplicity, and also because true indifference in most random $d$-Euclidean preference profiles will occur with probability zero.

    \begin{definition}
        If all profiles $\Pi\in\mathcal{P}_{A,I}$ are Euclidean of dimension $d$, then we say $d$ is \emph{sufficient} for $\mathcal{P}_{A,I}$.
    \end{definition}

For a profile $\Pi$, we use $\text{Euclidean}(\Pi)$ to denote a Euclidean preference model that optimally approximates $\Pi$.

\section{Three Questions}
\label{sec:three-questions}

\subsection{Are all profiles Euclidean?}
\label{sec:q1}

The first question was answered by \textcite{bogomolnaia2007}, who identified the minimum dimensionality required to represent all possible profiles of a given size.

    \begin{theorem}[\citeauthor{bogomolnaia2007} \citeyear{bogomolnaia2007}]
        Dimensionality $d$ is sufficient for $A$, $I$ if and only if $d\geq M$ where either  $M = \min\{I-1,A-1\}$ or $M = \min\{I, A-1\}$, depending on the values of $A$ and $I$.
    \end{theorem}
    
    \begin{proof}
        See \textcite{bogomolnaia2007}, Theorem 8, Proposition 13 and Proposition 15.
    \end{proof}

Thus, the answer to the first question is no, not all profiles are $d$-Euclidean. If $d < \min\{I-1, A-1\}$, then there exists at least one profile $\Pi\in\mathcal{P}_{A,I}$ which cannot be losslessly represented with a $d$-Euclidean model.

Is this really that big a deal? Maybe the profiles that are not $d$-Euclidean are only a small number of pathological edge cases that are unlikely to be encountered in the real world. The next question asks whether this is the case.

\subsection{How common are non-Euclidean profiles?}
\label{sec:q2}

For a given $d<\min\{I-1,A-1\}$, what proportion of preference profiles in $\mathcal{P}_{A,I}$ are not $d$-Euclidean? Given the NP-hardness of recognizing whether a given profile is Euclidean, a precise answer to this question is likely intractable. However, \textcite{bogomolnaia2007} defined three classes of pathological sub-profiles that, if present, cause a profile to not be $d$-Euclidean for some $d$. If we calculate the probability that a pathological sub-profile arises in a profile constructed uniformly at random---that is, in an \textit{impartial culture} \cite{ergecioglu2013}---this probability is precisely the proportion of profiles that exhibit that pathology (and hence are not $d$-Euclidean), which is a lower bound on the proportion of profiles that are not $d$-Euclidean. It is a lower bound because the set of all non $d$-Euclidean profiles is a superset of the set of profiles that contain this pathological sub-profile, and it is \textit{only} a lower bound because we are---to the best of our knowledge---considering only a subset of the pathologies that cause a profile to be non $d$-Euclidean.

It is not known whether the three classes of pathological sub-profiles identified by \textcite{bogomolnaia2007} exhaustively characterize the ways in which a profile can fail to be $d$-Euclidean---there may be other pathologies. %
In this work, we focus on one of these classes (which we call the \textit{circulant pathology}), as of the two other classes they consider, one requires the possibility of ties between alternatives \cite[Ex. 6]{bogomolnaia2007}, (which is a different ``ground truth'' preference model to that which we consider) and the other requires that $A \geq 2^I$ \cite[Ex. 14]{bogomolnaia2007} (which we deemed implausible for the settings we are interested in, such as recommender systems, where $I \gg 100$). Further, while we were able to derive or bound the probability of each class occurring in isolation, calculation of the joint probabilities did not appear tractable, and considering the circulant pathology alone is sufficient to produce non-trivial results. We emphasize that the bounds we derive apply more broadly to the phenomenon of non $d$-Euclidean profiles, because the set of such profiles is a superset of those that contain the circulant pathology.

    \begin{definition}[circulant pathology]
        A circulant pathology of size $k$ is a preference sub-profile consisting of $k$ alternatives $a_1,\dots,a_k$ and $k$ individuals $1,\dots,k$ such that
        $$
        \begin{aligned}
        & a_1 & >_1\ \ & a_2 & >_1\ \ & \dots & >_1\ \ & a_{k-1} & >_1\ \ & a_k\\ 
        & a_2 & >_2\ \ & a_3 & >_2\ \ & \dots & >_2\ \ & a_{k}   & >_2\ \ & a_1\\ 
        &\ \vdots&&\ \vdots&&&&\ \vdots&&\ \vdots\\ 
        & a_k & >_k\ \ & a_1 & >_k\ \ & \dots & >_k\ \ & a_{k-2} & >_k\ \ & a_{k-1}.
        \end{aligned}
        $$
    \end{definition}

Note that for each instance of the circulant pathology, there is a unique \textit{circular permutation} of $k$ alternatives that is consistent with all the sub-preferences involved in the pathology. (A circular permutation is a unique ordering of the alternatives on a circle.) We refer to each of these sub-preferences as \textit{necessary sub-preferences}.

    \begin{theorem}[\citeauthor{bogomolnaia2007} \citeyear{bogomolnaia2007}]
        If a profile $\Pi\in\mathcal{P}_{A,I}$ contains a circulant pathology of size $k$ as a sub-profile, then $\Pi$ is not $d$-Euclidean for any $d \leq k-2$.
    \end{theorem}

    \begin{proof}
        See \textcite{bogomolnaia2007}, Proposition 7 and Proposition 13.
    \end{proof}

Deriving the exact probability that a circulant pathology of size $k$ arises in a randomly generated profile appears difficult, but there is some related work. For example, consider a particular circulant pathology constructed using fixed subset of all $A$ alternatives, consistent with a fixed circular permutation of those alternatives. The set of all $A!$ possible preferences can be partitioned into those that could play the role of individual $1$, those that could be individual $2$ etc., and those that cannot be part of the pathology because the order in which the $k$ alternatives appear does not match any of the necessary sub-preferences. The probability that this specific pathology arises is the probability that when choosing $I$ preferences independently and uniformly at random, we choose at least one from all but the last part in this partition. This can be framed as the probability of completing a particular row on a bingo card within $I$ draws (with replacement). It is also closely related to the \emph{coupon collector problem} \cite{neal2008} in probability theory. However, results for these problems are not easily generalized to the situation in which every possible subset of alternatives of size $k\geq d+2$ might be used to construct the pathology.

From another angle, we might start from the probability that $k$ randomly chosen preferences contain a common sub-preference over any subset of $k$ alternatives, and then hope to adjust the probability to account for the circulant offset. The closest work to this is a set of papers on the longest common subsequences in random words or permutations \cite{bukh2014,houdre2018,houdre2022}, however so far this work is limited to the setting of $k=2$.

Whilst we cannot compute it exactly, we can bound from below the probability by restricting ourselves to a subset of the ways in which the pathology might be constructed. This brings the problem within reach of an approach similar to the bingo framing described above.

    \begin{theorem}[lower bound on probability of circulant pathology]
        \label{thm:probability-of-circulant-pathology}
        Let $A$, $I$, and $d$ be fixed positive integers such that $d<\min\{I,A-1\}$, and $\P(C)$ be the probability that a profile chosen uniformly from $\mathcal{P}_{A,I}$ contains a circulant pathology of size $k\geq d+2$. Then,
        $$\P(C) \geq 1 - \left(1 - \sum_{k=d+2}^I B_k \right) ^{ \left\lfloor \frac{A}{d+2} \right\rfloor }, $$
        where
        \resizebox{.89\linewidth}{!}{$B_k=\binom{I}{k} \genfrac\{\}{0pt}{0}{k}{d+2} (d+2)! \left(\frac{1}{(d+2)!}\right)^k \left(1 - \frac{d+2}{(d+2)!}\right)^{I-k}$}
        and $\genfrac\{\}{0pt}{1}{k}{d+2}$ denotes a Stirling number of the second kind.
    \end{theorem}

    \begin{proof}
        We build up the bound step by step. First, consider the version of the pathology that involves a specific subset of $d+2$ alternatives $a_1,\dots,a_{d+2}$, and is consistent with a specific circular permutation of those $d+2$ alternatives. For this version of the pathology to arise, we need at least one individual to have each of the following sub-preferences.
        $$\begin{aligned}
        &a_1 & >\ \  & a_2 & >\ \  & \dots & >\ \  & a_{d+2}\\ 
        &a_2 & >\ \  & a_3 & >\ \  & \dots & >\ \  & a_1\\ 
        &\ \vdots&&\ \vdots&&&&\ \vdots\\ 
        & a_{d+2} & >\ \  & a_1 & >\ \  & \dots & >\ \  & a_{d+1}
        \end{aligned}$$
        Recall that we call these \textit{necessary sub-preferences}. The probability that this version of the pathology arises in a profile chosen uniformly at random from $\mathcal{P}_{A,I}$ is 
        \begin{multline*}
            \underbrace{
             \binom{I}{k}
             \genfrac\{\}{0pt}{0}{k}{d+2}
             (d+2)!
            }_{\substack{\textsf{\# ways to choose $k$ individuals from $I$,} \\ \textsf{partition them into $d+2$ non-empty parts,} \\ \textsf{and assign each part to a necessary}\\ \textsf{sub-preference}}}\\ 
            \times
            \underbrace{
             \left(\frac{1}{(d+2)!}\right)^k
            }_{\substack{\textsf{probability those $k$ individuals}\\ \textsf{randomly get assigned}\\ \textsf{preferences that have those}\\ \textsf{necessary sub-preferences}}}
            \times
            \underbrace{
             \left(1-\frac{d+2}{(d+2)!}\right)^{I-k}
            }_{\substack{\textsf{probability the other $I-k$}\\ \textsf{individuals randomly get assigned}\\ \textsf{preferences that lack those}\\ \textsf{necessary sub-preferences}}}
        \end{multline*}
        For brevity, we will use $B_k$ to refer to this binomial-like expression. Note that $\genfrac\{\}{0pt}{1}{k}{d+2}$ denotes a Stirling number of the second kind, and can be defined as the number of ways to partition a set of $k$ objects into $d+2$ non-empty subsets. The number of individuals involved in the pathology, $k$, must be at least $d+2$, and can be as high as $I$. Further, because we explicitly require that the $I-k$ individuals not involved in the pathology do not have any of the necessary sub-preferences, the events where different numbers of individuals are involved are mutually exclusive. Thus, we can sum over the alternate choices of $k$ to get
        $$\sum_{k=d+2}^I B_k,$$
        which is the probability that this particular version of the pathology arises for any $k$.
        
        Ideally, we would be able to generalize this expression to account for other versions of the pathology that involve the same fixed set of alternatives, but are consistent with different circular permutations of those alternatives. This is not easy to do because we have already counted the necessary sub-preferences for these versions of the pathology among the non-necessary sub-preferences in the expression of $B_k$. Put another way, the occurrences of versions of the pathology with different reference permutations are not independent.
        
        However, versions of the pathology constructed using disjoint subsets of alternatives \textit{do} occur independently. This is because the subsets of alternatives can be permuted within a preference order without affecting the relative positions of alternatives that are not part of that subset.

        \begin{figure}[t]
            \includegraphics[width=0.99\columnwidth]{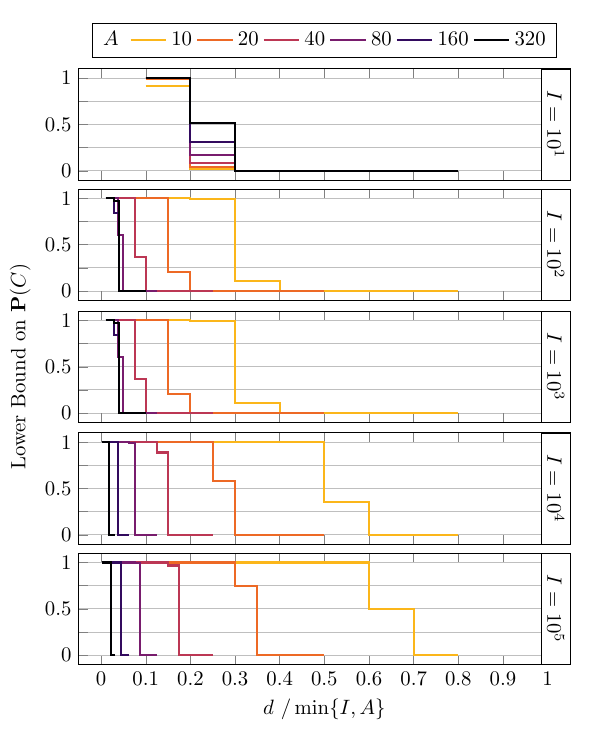}
            \caption{Lower bound on $\P(C)$, and hence the probability that a profile chosen uniformly at random is not $d$-Euclidean, for various $d$, $A$ and $I$.}
            \label{fig:circulant-profiles}
        \end{figure}
        
        Thus, we can generalize the expression to allow for the fact that there are multiple disjoint subsets of alternatives from which this pathology can arise. Each subset must be of size $d+2$, so at most we could specify $\left\lfloor \frac{A}{d+2} \right\rfloor$ disjoint subsets. Using the independence property from the previous paragraph, the probability increases to
        $$1 - \underbrace{\left(1 - \sum_{k = d+2}^I B_k \right) ^ {\left\lfloor \frac{A}{d+2} \right\rfloor}}_{\substack{\textsf{probability that none of these}\\ \textsf{$\left\lfloor A/(d+2) \right\rfloor$ versions of the}\\ \textsf{pathology occur}}}.$$
        This expression is the probability that at least 1 of each of $\left\lfloor A/(d+2) \right\rfloor$ versions of the pathology, each constructed using a disjoint set of alternatives, occurs. However, it does not account for all possible subsets of alternatives that could be used, or all possible circular permutations of those subsets. Thus, it is a lower bound on $\P(C)$, the probability that any version of the pathology is contained in a randomly chosen profile.
    \end{proof}

We can now numerically evaluate this expression for a range of values of $A$, $I$ and $d$ to see what the probability (or proportion) is in real terms (Figure \ref{fig:circulant-profiles}). When $d$ is a lot smaller than $A$ and $I$, almost all profiles are not $d$-Euclidean, and for any fixed $d$ and $A$ this proportion appears to approach $1$ as $I \to \infty$. Increasing $d$ appears to be quite effective at reducing the proportion of non-Euclidean profiles (there is a lot of room in a high-dimensional space).

We conclude that in some circumstances, almost all profiles are not $d$-Euclidean. So what? If we \textit{approximate} them with a Euclidean preference model, is the approximation error big enough to worry about?

\subsection{How large is the expected error?}
\label{sec:q3}

    \begin{figure}
        \includegraphics[width=0.99\columnwidth]{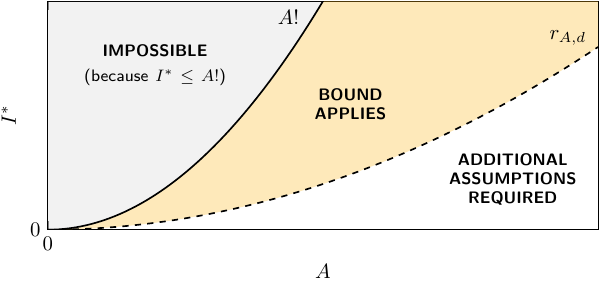}
        \caption{A (not to scale) diagram indicating the values of $I^*$ and $A$ for which the bound on the expected error given in Theorem \ref{thm:error} applies. When $I^* < r_{A,d}$, additional assumptions would be required to produce a non-trivial lower bound, because the expected error could be as low as zero depending on the profile $\Pi$.}
        \label{fig:bound-applicability}
    \end{figure}

There are different ways to quantify the error of a $d$-Euclidean approximation to an arbitrary preference profile. The most natural approach may be to count the minimal number of swaps needed to change the given profile into one that is $d$-Euclidean. However, this intuitive approach seems difficult given the findings of \textcite{peters2016}, who proved that the task of finding such best-approximations is, in general, NP-hard. Instead, we consider the question from the perspective of individual preferences, rather than complete profiles, and use the following setup:
\begin{enumerate}
    
    \item Take an arbitrary profile $\Pi \in \mathcal{P}_{A,I^*}$ consisting of $\UniquePref$ unique preferences. (It doesn't matter which profile it is — we assume we don’t know.)

    \item Approximate $\Pi$ as well as possible in a $d$-dimensional Euclidean model. Call this model $\text{Euclidean}(\Pi)$. (It doesn't matter how good the approximation algorithm is, our bound assumes it finds the best possible fit.)

    \item Observe a new preference $\pi$ generated uniformly from among all $A!$ preferences. (It could duplicate one of the existing $\UniquePref$ preferences. The new preference $\pi$ is generated \textit{after} the model $\text{Euclidean}(\Pi)$ is fit. Knowledge of $\pi$ cannot inform $\text{Euclidean}(\Pi)$, otherwise one could trivially choose a model in which $\pi$ was representable.)

    \item Let $\hat{\pi}$ be the preference representable in $\text{Euclidean}(\Pi)$ that minimizes $m(\hat{\pi}, \pi)$ for some error measure $m$.
    
\end{enumerate}
That is, $\hat{\pi}$ is the closest possible approximation to $\pi$ in such a Euclidean preference model. As our measure of error $m$, we use the Kendall tau distance (or bubble-sort distance):
$$m(\pi, \pi')=\text{\# pairwise disagreements between }\pi\text{ and }\pi'.$$
Equivalently, $m(\pi, \pi')$ can be defined as the minimum number of adjacent swaps required to transform $\pi$ to $\pi'$. The Kendall tau distance is well-established \cite{kumar2010} and has been widely used as a metric between preferences in the social choice literature \cite{obraztsova2012,obraztsova2013,anand2021}. We are interested in both $\E[m(\pi,\hat{\pi})]$ and $\E[m(\pi, \hat{\pi})] / \binom{A}{2}$, which is the expected number of adjacent swaps required as a proportion of the maximum possible number of swaps between any two preferences. Intuitively, $\E[m(\pi, \hat{\pi})] / \binom{A}{2}$ can also be interpreted as the probability that two alternatives chosen uniformly at random will be ranked differently by $\pi$ and $\hat{\pi}$. Given the lack of tractable algorithms or formulae for optimally approximating arbitrary profiles with Euclidean models (see Section \ref{sec:introduction}), we do not compute these expectations exactly, and instead derive lower bounds.

The intuition for our approach is as follows. For any choice of $A$ and $d$, there will be some maximum number of unique preferences (or permutations) $r_{A,d} \leq A!$ that can be simultaneously represented in Euclidean space of dimension $d$. In what follows, we will abbreviate $r_{A,d}$ to $r$ (the dependence on $A$ and $d$ is implied) and in order to bound the expected error will assume that $\UniquePref \geq \MaxPref$ because, in this case, the probability that $\pi$ is representable is fixed (Figure \ref{fig:bound-applicability}). Assume that we have an overestimate of $\MaxPref$, say $\hat{\MaxPref}$. Then, if we ``distribute'' these $\hat{\MaxPref}$ ``representable'' preferences ``evenly'' throughout the set of all possible preferences, this will minimize the expected distance between $\pi$ and $\hat{\pi}$, the nearest of the $\hat{\MaxPref}$ ``representable'' preferences. We use this minimized distance (or more precisely, distribution of distances) to produce a lower bound for $\E[m(\pi,\hat{\pi})]$. It is only a lower bound because $\hat{\MaxPref}$ may overestimate $\MaxPref$.

Along with an upper bound $\hat{\MaxPref}$, this approach requires a structure on the set of all $A!$ preferences that allows us to ``evenly distribute'' the $\hat{\MaxPref}$ preferences that might be representable. For this, we use the \emph{permutohedron} (Figure \ref{fig:permutohedron}), a high dimensional polytope where each vertex corresponds to a permutation, and edges between vertices correspond to single swap operations between adjacent elements. (There is a bijection between permutations and preferences.) For example, the permutohedron of order 4 will have a vertex corresponding to the permutation \texttt{1234} that is connected to three other vertices: \texttt{2134}, \texttt{1324}, and \texttt{1243}, because they are the permutations that can be reached by applying a single adjacent swap to \texttt{1234} \cite{gaiha1977}. Our overestimate for $\hat{\MaxPref}$ is given by the following lemma.

    \begin{figure}
        \centering
        \includegraphics[width=0.99\columnwidth]{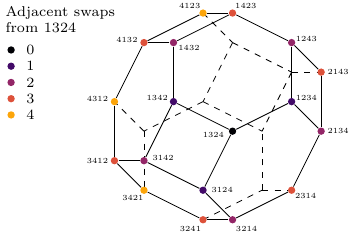}
        \caption{The permutohedron of order 4.}
        \label{fig:permutohedron}
    \end{figure}

    \begin{lemma}[upper bound on $\MaxPref$]
        \label{lemma:representable-proportion-of-preferences}
        Let $\MaxPref$ be the maximum number of $A!$ unique preferences over $A$ alternatives that are simultaneously representable in a $d$-dimensional Euclidean preference model. If $\UniquePref \geq \MaxPref$, then
        $$\MaxPref \leq \hat{\MaxPref} = \left(1 - \P(\CircPref)\right)A!,$$
        where $\P(\CircPref)$ is the proportion of A! unique preferences over $A$ alternatives that \emph{cannot} be simultaneously represented due to the circulant pathology. Explicitly, let $\Index{n}$ be the event that a random permutation of integers $1,\dots,A$ places the integer $n$ ($n\leq A$) within the first $A-d-n$ positions, and $\ \comp{\Index{n}}$ denote the complement of $\Index{n}$. Then $\P(\CircPref)$ can be written as
        $$\P(\CircPref) = \P(\Index{1}) + \P(\comp{\Index{1}})\P(\CircPref\mid \comp{\Index{1}}),$$
        where the conditional probability is defined recursively, as follows. For $N<A-d-2$,
        \begin{multline*}
            \P\left(\CircPref\mid \bigcap_{n = 1}^N \comp{\Index{n}}\right) =\P\left(\Index{N+1}\mid \bigcap_{n = 1}^N \comp{\Index{n}}\right)\\
            + \P\left(\comp{\Index{N+1}}\mid \bigcap_{n = 1}^N \comp{\Index{n}}\right) \P\left(\CircPref\mid \bigcap_{n = 1}^{N+1} \comp{\Index{n}}\right),
        \end{multline*}
        and for $N=A-d-2$,
        $$\P\left(\CircPref\mid \bigcap_{n = 1}^N \comp{\Index{n}}\right)=\P\left(\Index{A-d-1}\mid \bigcap_{n = 1}^{A-d-2} \comp{\Index{n}}\right).$$
    \end{lemma}

    \begin{proof}
        The intuition for the proof is as follows. We want to start with the full set of $A!$ possible preferences, and then ``ban'' the minimum number of preferences required to guarantee that the circulant pathology does not arise. What is this minimum number?
        
        We need to only review circulant pathologies of size $d+2$ (the smallest possible), as any circulant pathologies of size $k>d+2$ would contain a size $d+2$ pathology as a sub-profile.
        
        For every combination of $d+2$ alternatives, there are $(d+1)!$ circular permutations of those alternatives. Each of these circular permutations characterizes a way in which the circulant pathology might arise.
        
        For each of these circular permutations of $d+2$ alternatives, there are $d+2$ necessary sub-preferences over those alternatives that are consistent with the circular permutation. To avoid the pathology, we need to ban one of these necessary sub-preferences for every combination and circular permutation of $d+2$ alternatives.
        
        At first, determining how many (complete) preferences are affected by this set of sub-preference bans appears difficult, because a preference may be consistent with more than one of the banned sub-preferences, and these banned sub-preferences may involve some of the same alternatives. However, it turns out that the task is equivalent to computing the probability of a particular event related to the positions of some specific alternatives within a random permutation of all alternatives, and that this computation is tractable.
        
        Consider each set of $d+2$ sub-preferences (corresponding to each combination and circular permutation of $d+2$ alternatives) that could form the circular pathology. We must ban one of these $d+2$ sub-preferences. Without loss of generality, we choose to ban the sub-preference that ranks most highly the alternative with the minimum index. For example, of the set of sub-preferences
        $$\begin{aligned}
        & a_6    & >\ \ \ & a_{13} & >\ \ \ & a_2\\ 
        & a_{13} & >\ \ \ & a_2    & >\ \ \ & a_6\\
        & a_2    & >\ \ \ & a_6    & >\ \ \ & a_{13}
        \end{aligned}$$
        we would ban the third sub-preference. The reason that this convention can be used without loss of generality is that we only care about the minimum number of sub-preferences that must be banned, and this convention is precise in that it targets exactly that minimum number. (Any other convention that had the same property could equally be used.)
        
        With this convention, the proportion of preferences that are banned is equal to the probability of the event $\CircPref$ that a random preference over the alternatives $a_1,\dots,a_A$ ``positions a low index alternative sufficiently near the top''. More precisely,
        $$\CircPref = \bigcup_{n=1}^{A-d-1} \Index{n},$$
        where $\Index{n}$ is the event that a random permutation of the alternatives $a_1,\dots,a_A$ positions $a_n$ within the first $A-d-n$ positions of the permutation. If (and only if) $\CircPref$ occurs, there will be a low-index alternative ranked sufficiently highly to ensure that there are enough high-index alternatives ranked below it to induce at least one of the banned sub-preferences.
        
        It now remains to compute $\P(\CircPref)$. There is no closed form expression, but it can be written down using a recursive formula that corresponds to repeated application of the law of total probability. Specifically,
        $$\begin{aligned}
        \P(\CircPref) &= \P(\Index{1})\underbrace{\P(E\mid\Index{1})}_{=1} + \P(\comp{\Index{1}})\P(\CircPref\mid \comp{\Index{1}}) \\[4pt]
        &= \P(\Index{1}) + \P(\comp{\Index{1}})\P(\CircPref\mid \comp{\Index{1}})
        \end{aligned}$$
        where the conditional probability is defined recursively, as follows. For $N<A-d-2$,
        \begin{multline*}
            \P\left(\CircPref\mid \bigcap_{n = 1}^N \comp{\Index{n}}\right) =\P\left(\Index{N+1}\mid \bigcap_{n = 1}^N \comp{\Index{n}}\right)\\
            + \P\left(\comp{\Index{N+1}}\mid \bigcap_{n = 1}^N \comp{\Index{n}}\right) \P\left(\CircPref\mid \bigcap_{n = 1}^{N+1} \comp{\Index{n}}\right),
        \end{multline*}
        and for $N=A-d-2$,
        $$\P\left(\CircPref\mid \bigcap_{n = 1}^N \comp{\Index{n}}\right)=\P\left(\Index{A-d-1}\mid \bigcap_{n = 1}^{A-d-2} \comp{\Index{n}}\right).$$
        The above expressions look a bit intimidating but the basic idea---the law of total probability---is quite simple. We compute the probability that $a_1$ is in the top $A-d-1$ positions, or if not, that $a_2$ is in the top $A-d-2$ positions, or if not, that $a_3$ is in the top $A-d-3$ positions, and so on. The relevant probabilities are analytically simple to compute. Explicitly, for $N \leq A-d-2$,
        $$\begin{aligned}
        \P\left(\Index{N+1}\mid \bigcap_{n = 1}^N \comp{\Index{n}}\right) &= 1 - \P\left(\comp{\Index{N+1}}\mid \bigcap_{n = 1}^N \comp{\Index{n}}\right)\\ &=1 - \frac{ (A-N-1)_{A-N-d-1} }{ (A-N)_{A-N-d-1} }\\ &= 1 - \prod_{k=d+2}^{A-N} \frac{ k-1 }{ k }. 
        \end{aligned}$$
        The second equality comes from a combinatorial argument: to deduce the falling factorial in both the numerator and the denominator, write out how many alternatives are ``available'' to populate each position in the permutation given the conditioning constraints, and multiply them together.
    \end{proof}
    
How small is $\hat{\MaxPref}$, in real terms? Figure \ref{fig:circulant-preferences} plots $\hat{\MaxPref}/A!$ for a variety of values of $d$ and $A$. In most cases, only a small proportion of preferences can be simultaneously represented, and it is only as $d$ nears $A$ that the proportion increases. So how much error should we expect when approximating such preferences in a Euclidean preference model?

    \begin{figure}
        \includegraphics[width=0.99\columnwidth]{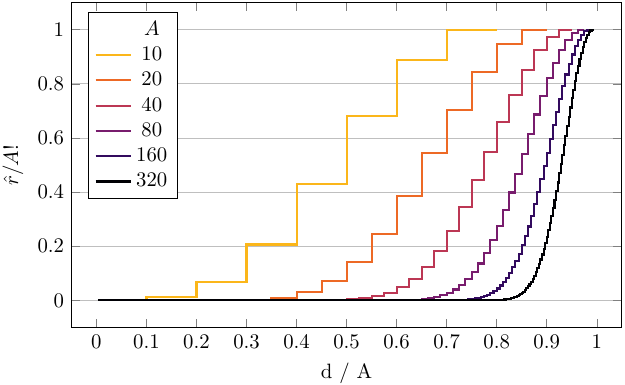}
        \caption{Upper bound on the proportion of all $A!$ possible preferences that can be simultaneously represented in a $d$-Euclidean model, for various $A$, $d$, and valid when $\UniquePref \geq \MaxPref$.}
        \label{fig:circulant-preferences}
    \end{figure}

  \begin{figure*}
        \centering
        \includegraphics[width=\textwidth]{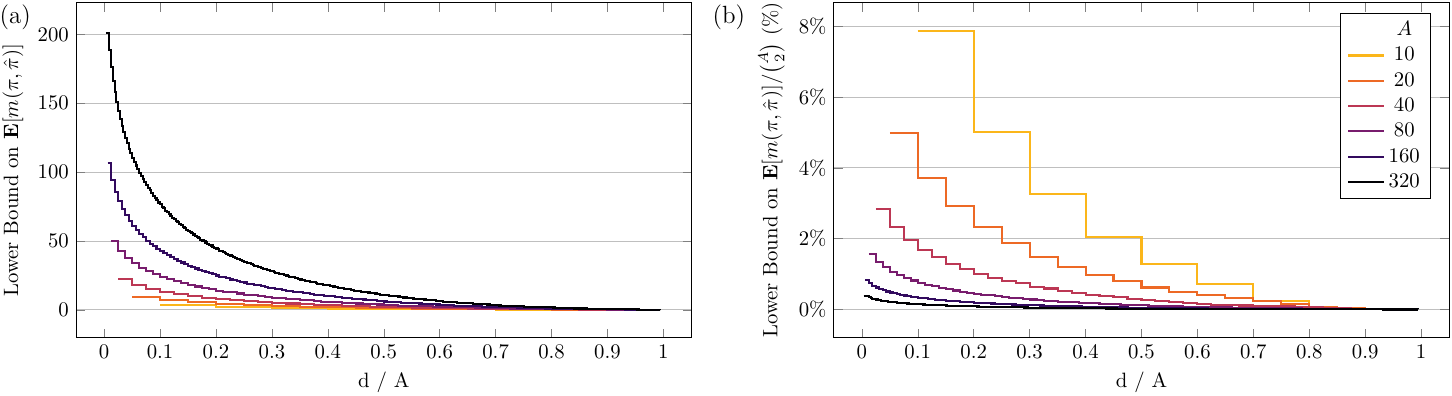}
        \caption{(a) The lower bound on $\E[m(\pi,\hat{\pi})]$ and (b) the lower bound on $\E[m(\pi,\hat{\pi})]/\binom{A}{2}$, both for $\UniquePref \geq \MaxPref$ and various $A$, $d$.}
        \label{fig:error}
  \end{figure*}

    \begin{theorem}[lower bound on expected error]
        \label{thm:error}
        Let $A$, $d$ be fixed positive integers such that $d<A-1$, $\Pi \in \mathcal{P}_{A,I^*}$ consist of $\UniquePref$ unique preferences, $\pi$ be a preference chosen uniformly at random from the set of $A!$ possible preferences, $\hat{\pi}$ be the nearest preference to $\pi$ that is representable in $\text{Euclidean}(\Pi)$ (that is, the representable preference that can be reached in the fewest number of adjacent swaps), and $K$ be a positive integer such that $K\leq \binom{A}{2}$. If $\UniquePref \geq \MaxPref$, then
        $$\E[m(\pi,\hat{\pi})] \geq \sum_{k=0}^K \frac{\left(A! - n_{k,A} \right)_{\hat{\MaxPref}}}{\left(A!\right)_{\hat{\MaxPref}}}\mathbf{1}(\hat{\MaxPref} < A! - n_{k,A}),$$
        where $(\cdot)_{\hat{\MaxPref}}$ denotes a falling factorial, $\mathbf{1}(\cdot)$ the indicator function, and $n_{k,A}=\min\{(A-1)^k,\ A!\}$.
    \end{theorem}

    \begin{proof}
        First we assume that $\UniquePref \geq r$, because without this assumption, it is possible that all preferences in our profile are simultaneously representable, so the lower bound on our measure of error would be zero. Without further information about the composition of the particular profile $\Pi$, a lower bound on $\E[m(\pi,\hat{\pi})]$ can only be larger than zero if there are more unique preferences than can be simultaneously represented in a $d$-dimensional Euclidean model---that is, when $\UniquePref \geq r$.
    
        The distribution of the $\hat{\MaxPref}$ ``representable'' preferences over the vertices of the permutohedron that would minimize the expectation $\E[m(\pi,\hat{\pi})]$ is the uniform distribution. The intuition for this is a symmetry argument. All vertices on the permutohedron are topologically equivalent to one another — none are better connected than the others — so there are no free lunches. Any distribution of the ``representable'' preferences over the permutohedron that deviated from the uniform distribution by tending to place those preferences nearer to some vertices would tend to move them farther away from many more vertices.
        
        Thus, to lower bound the expectation, we select $\hat{\MaxPref}$ vertices uniformly at random (without replacement), at which to position the ``representable'' preferences. This gives a function $F$ that bounds from above the distribution function of $m(\pi,\hat{\pi})$. For any non-negative integer $k$,
        $$F(k)= 1 - \underbrace{\frac{\left(A! - n_{k,A} \right)_{\hat{\MaxPref}}}{\left(A!\right)_{\hat{\MaxPref}}}\mathbf{1}(\hat{\MaxPref} < A! - n_{k,A})}_{\substack{\textsf{probability that none of the $\hat{\MaxPref}$ ``representable''}\\ \textsf{preferences fall on the $n_{k,A}$ ``reachable'' vertices}}},$$
        where $(\cdot)_{\hat{\MaxPref}}$ denotes a falling factorial, $\mathbf{1}(\cdot)$ the indicator function, and $n_{k,A}=\min\{(A-1)^k,\ A!\}$ is an upper bound on the number of unique preferences that are reachable within $k$ swaps. The formula for $n_{k,A}$ arises from the fact that each additional swap makes accessible at most a factor of $A-1$ additional unique preferences, though this is a very loose upper bound because many of those will be reachable with fewer swaps. The minimum is used because there are only $A!$ vertices on the permutohedron.
        
        The form of $F$ is analogous to the exact distribution function of $m(\pi,\hat{\pi})$ under our assumed uniform distribution on the positions of the representable preferences---namely, one minus the probability mass function of a hypergeometric random variable evaluated at zero, where the random variable corresponds to the number reachable vertices selected in $r$ uniformly random draws, without replacement. However, $F$ uses $\hat{\MaxPref}$ in place of $\MaxPref$ and $n_{k,A}$ in place of the true number of vertices that are reachable within $k$ swaps (Figure \ref{fig:permutohedron}). Because of these substitutions, as well as the assumption of the uniform distribution, we have 
        $$\P(m(\pi,\hat{\pi}) \leq k) \leq F(k).$$
        Now, using the well known formula that expresses the expectation of a non-negative random variable in terms of its distribution function, we can turn this into a bound on the expectation. For any non-negative integer $K<\binom{A}{2}$,
        $$\begin{aligned}
            \E[m(\pi,\hat{\pi})]
                &=\sum_{k=0}^\infty \P(m(\pi,\hat{\pi}) > k)\\
                &=\sum_{k=0}^\infty 1 - \P(m(\pi,\hat{\pi}) \leq k)\\ 
                &\geq \sum_{k=0}^\infty 1 - F(k) \geq \sum_{k=0}^K 1 - F(k),
        \end{aligned}$$
        which is the bound in the theorem.
    \end{proof}

How high is the bound in real terms? Figure \ref{fig:error}(a) plots some values of the lower bound on $\E[m(\pi,\hat{\pi})]$, and Figure \ref{fig:error}(b) plots the equivalent scaled values, $\E[m(\pi,\hat{\pi})]/\binom{A}{2}$. In both cases, we can see that expected error becomes more severe as $d/A \to 0$, but that the bound on the scaled loss appears less severe for larger values of $A$.

We reiterate that while the bounds in Lemma \ref{lemma:representable-proportion-of-preferences} and Theorems~\ref{thm:probability-of-circulant-pathology} and \ref{thm:error} are derived using only one class of pathological sub-profiles (the circulant pathology), they apply more broadly to the general phenomenon of non $d$-Euclidean profiles, because the set of such profiles is a superset of those that contain a circulant pathology.

\section{Conclusions}
\label{sec:conclusions}

We have proven theoretical lower bounds on the proportion of preference profiles of $I$ individuals over $A$ alternatives that are not $d$-Euclidean, and on the expected error when representing an arbitrary preference in a Euclidean model. Ultimately, these bounds show that when $d$ is small relative to $I$ and $A$ ($d \ll \min\{I, A\}$), almost all preference profiles are not $d$-Euclidean and, with preferences from an impartial culture, the expected error when approximating an additional, arbitrary preference in the Euclidean model can in some cases be at least 7\% of the maximum possible variation between any two preferences, as measured by the Kendall tau distance.

\subsection{Implications}

Our lower bound on the expected error has been proven for $\UniquePref \geq \MaxPref$, and is largest when $d \ll \min\{I, A\}$. These conditions may be met, for example, in some social choice models of national elections, or in models of reputation systems where there are a large number of Sybils. In such settings, our bound shows that close approximation of true underlying preference profiles may not always be expected. The magnitude of error will depend on the particular profile, and can be reduced by increasing the dimensionality, $d$. Our bound provides a way to quantify the trade-off between dimensionality and accuracy, and hence inform the choice of $d$.

Settings where $d \ll \min\{I, A\}$ are also common in deep learning. For example, text embeddings used for natural language processing commonly have $d \approx 10^3$ dimensions \cite[Ch. 6.8]{jurafsky2020}, but the English language has $\approx 10^5$ unique words \cite{nagy1984}, and the number of sentences or paragraphs will be orders of magnitude larger. Similarly, embeddings in recommender systems have up to $d \approx 10^5$ dimensions, but large platforms can serve $I \approx 10^9$ users and host $A \approx 10^{11}$ items of content \cite{satuluri2020}. To the extent that there are true, underlying preferences (or more generally, ordinal relationships) in such domains, our results suggest that these relationships may not be able to be closely approximated by such relatively low dimensional embeddings if the use of those embeddings assumes the Euclidean structure.

It is important that we are precise about the uses of embeddings in deep learning systems to which these bounds apply. There are at least two common use cases potentially affected.

\paragraph{Proximity-Based Ranking.} Embeddings are commonly compared using cosine similarity, which induces the same ordinal relationships as the Euclidean metric when applied to normalized vectors (see Section \ref{sec:introduction}). Thus when cosine similarity or Euclidean distance is used to rank the similarity of embeddings to an ideal point, such as when ranking the top-$k$ most preferred items for a user in a recommender system, the bounds suggest a limit on how accurately these ordinal relationships can be recovered.

\paragraph{SoftMax.} Embeddings are also used to recover a probability distribution over a set of alternatives. Often this is done by taking the inner product between the embedding of each alternative with a fixed reference vector, and then converting these values into probabilities using the SoftMax function \cite[Sec. 6.2.2.3]{goodfellow2016}. Geometrically, this corresponds to projecting the embeddings onto the reference vector, so the ordering of the alternatives from most probable to least probable in the resulting distribution is the same as the ordering implied by the preferences of an individual in a Euclidean preference model where the individual's ideal point is placed sufficiently distant in the direction of the reference vector. In this context, our bounds suggest that there will be pairs of alternatives for which the model is mistaken about which is more likely.

The extent to which these bounds apply to embeddings when they are used in their original context---that is, when used as inputs to a neural network with which they were jointly trained---remains an open question. It is possible that a trained network interprets the relationships between embeddings using an approximation to the Euclidean metric, in which case similar limits on accuracy may apply. But it is also possible that the network effectively memorizes the exceptions to the ordinal relations implied by the underlying spatial model, and is thus able to compensate for any error.

\subsection{Future Work} 

Our results pave the way for a research agenda useful to inform the choice of the dimensionality of embeddings or Euclidean preference models. As discussed, the bound on the expected error has only been proven for cases where $\UniquePref \geq \MaxPref$, which cannot be guaranteed for all applications. Moreover, the (computational) complexity of the bounds prevents them from being evaluated for settings where both $I$ and $A$ are larger than about $10^3$. Future work should thus prioritize extending the bound on the expected error to cases where $\UniquePref < \MaxPref$, and developing simpler approximations to the bounds that can be evaluated efficiently for large $I$ and $A$.

Another important direction would be to explore the robustness of our results to different distributions over preferences, as well as to different ``ground truth'' models of preference. For example, our proofs assume a uniform distribution over preferences and profiles (an impartial culture), but in practice some will be more probable than others. It is not clear how this would affect the bounds. It may be possible to create algorithms that detect each of the known pathologies, and use these to assess their prevalence in real-world preference data. By process of elimination, such classifiers might also facilitate the discovery of as yet unknown pathologies.

It would also be useful to understand the extent to which these bounds may contribute to sub-optimal performance of recommender systems. For example, it may be the case that accurately modelling ordinal preferences ``at the top end'' is more important than accuracy lower down the preference rankings, and it might be possible to accurately model the ordering of the most-preferred items, even if the complete preference ordering is not representable.

Finally, we note that while our lower bounds are informative, they are not tight. It is possible that non $d$-Euclidean profiles are substantially more common, or that the error is substantially more severe, than the bounds themselves. Given the possible implications for the accuracy of a widely-used preference model, it would be valuable to improve the tightness of the lower bounds, and to derive meaningful upper bounds.

\section{Acknowledgments}
Luke Thorburn was supported by UK Research and Innovation [grant number EP/S023356/1], in the UKRI Centre for Doctoral Training in Safe and Trusted Artificial Intelligence (\url{safeandtrustedai.org}). Carmine Ventre acknowledges funding from the UKRI Trustworthy Autonomous Systems Hub (EP/V00784X/1).

\printbibliography

\end{document}